\documentclass[runningheads,a4paper]{llncs}

\usepackage[usenames]{color}
 
\usepackage{graphicx}
\usepackage[all, knot]{xy}
\usepackage{amsmath}
\usepackage{mathabx}

\usepackage{comment}
\usepackage{times}
 
\usepackage{url}

\begin{document}

\mainmatter  

\title{Representation, Justification and Explanation in a Value Driven Agent: An Argumentation-Based Approach} 

\titlerunning{Representation, Justification and Explanation in a Value Driven Agent}

%
%
\author{Beishui Liao\inst{1} \and Michael Anderson\inst{2} \and Susan Leigh Anderson\inst{3}} 
\authorrunning{B. Liao, M. Anderson and S. L. Anderson} 

\institute{ 
Zhejiang University, Hangzhou 310028, P.R. China
\and
University of Hartford, Hartford, CT, USA
\and 
University of Connecticut, Storrs, CT, USA
}

%

 

%
%

\toctitle{Lecture Notes in Computer Science}
\tocauthor{Authors' Instructions}
\maketitle

\begin{abstract}
Ethical and explainable artificial intelligence is an interdisciplinary research area involving computer science, philosophy, logic, the social sciences, etc. For an ethical autonomous system, the ability to justify and explain its decision making is a crucial aspect of transparency and trustworthiness. This paper takes a Value Driven Agent (VDA) as an example, explicitly representing implicit knowledge of a machine learning-based autonomous agent and using this formalism to justify and explain the decisions of the agent. For this purpose, we introduce a novel formalism to describe the intrinsic knowledge and solutions of a VDA in each situation. Based on this formalism, we formulate an approach to justify and explain the decision-making process of a VDA, in terms of a typical argumentation formalism, Assumption-based Argumentation (ABA). As a result, a VDA in a given situation is mapped onto an argumentation framework in which arguments are defined by the notion of deduction. Justified actions with respect to semantics from argumentation correspond to solutions of the VDA. The acceptance (rejection) of arguments and their premises in the framework provides an explanation for why an action was selected (or not). Furthermore, we go beyond the existing version of VDA, considering not only practical reasoning, but also epistemic reasoning, such that the inconsistency of knowledge of the VDA can be identified, handled and explained.    
\end{abstract}





\section{Introduction}
Ethical, explainable artificial intelligence is an increasingly active research area in recent years.  An autonomous agent should make decisions by determining ethically preferable behavior \cite{DBLP:conf/aies/LiaoST19, DBLP:journals/pieee/Anderson18}. Further, it is expected to provide explanations to human beings about how and why the decisions are made \cite{DBLP:conf/isaim/BaumHS18}. Explainable AI is an interdisciplinary research direction, involving computer science, philosophy, cognitive psychology/science, and social psychology \cite{DBLP:journals/ai/Miller19}. In recent years, different approaches have been proposed to provide explanations for autonomous systems, though most of them are still rather preliminary. For instance, \cite{Cocarascu2018} proposed an architecture combining Artificial Neural Networks and argumentation for solving binary classification problems, \cite{DBLP:conf/ijcai/ShihCD18} introduced an approach for explaining Bayesian network classifiers such that the classifiers are compiled into decision functions that have a tractable and symbolic form,  and \cite{DBLP:journals/corr/abs-1806-08055} proposed a human explanation model based on conversational data. 

While there are various types of explanations such as trace, justification and strategy, according to the empirical results reported in \cite{DBLP:journals/misq/YeJ95}, justification is the most effective type of explanation to bring about changes in user attitudes toward the system. In order to provide a justification, one needs to first have a formal representation of the knowledge that is used in the process of decision making of an autonomous agent. For a machine learning-based agent, unfortunately, the formal and logical knowledge may not always be self-evident. In such cases modeling justification-based explanation entails formally representing the intrinsic knowledge of the agent to permit its use for justification and explanation, as exemplified in \cite{Cocarascu2018} and \cite{DBLP:conf/ijcai/ShihCD18}. In this paper, we study this methodology by considering an ethical agent, called a Value Driven Agent (VDA) in \cite{DBLP:conf/aaai/AndersonAB17}, focusing on the following research question. 

\begin{description}
\item[Research question] \emph{How can we explicitly represent the implicit knowledge of a VDA and use it to provide formal justification and explanation for its decision-making?}
\end{description}

A VDA, as introduced in next section, uses inductive logic programming techniques to abstract a principle from a set of cases, and a decision tree to determine the ethical consequences of each action in the current situation.  Interestingly, there exists untapped implicit knowledge that can help provide an account as to why a VDA determines an action is considered ethically preferable. 

In existing literature, there are a number works on value-based practical reasoning, e.g. \cite{DBLP:journals/ail/Bench-CaponM17} and \cite{DBLP:conf/icail/Bench-CaponM19}, which are related to the ethical decision-making of a VDA. However, the existing work has not considered how preferences over actions can be induced from cases, nor how a logic-based formalism can be integrated with a machine leaning based approach. 

In addition, concerning the combination of machine leaning-based approaches and formal logic-based approaches, while some existing works are mainly for explaining classification, e.g. \cite{Cocarascu2018} and \cite{DBLP:conf/ijcai/ShihCD18}, we are more interested in a system that can reason about the state of the world, the actions of a VDA, and the ethical consequences of the actions.

With these ideas in mind, in this paper, we study an \textit{explainable} VDA by exploiting formal argumentation. 
The structure of this paper is as follows. Section 2 recalls some required basic notions in existing literature. Section 3 introduces a formalism for representing knowledge of a value driven agent. In Section 4, we present an argumentation-based justification and explanation approach. In section 5, we discuss related work. Finally, we offer our conclusions in Section 6. 

\section{Preliminaries}

In this section, we introduce some basic notions required for understanding what follows, including those about Value Driven Agents and formal argumentation. 

\subsection{Value Driven Agent}

According to \cite{DBLP:conf/aaai/AndersonAB17}, A VDA is defined as an autonomous agent that decides its next action using an ethical preference relation over actions, termed a \textit{principle}, that is abstracted from a set of cases using inductive logic programming techniques. A \textit{case-supported, principle based approach} (CPB) uses a representation scheme that includes \textit{ethically relevant features} (e.g. harm, good, etc.) and their incumbent \textit{prima facie} duties to either minimize or maximize them (e.g. minimize harm, maximize good), \textit{actions} characterized by integer degrees of presence or absence of ethically relevant features (and so, indirectly, the duties it satisfies or violates), and \textit{cases} comprised of the differences of the corresponding duty satisfaction/violation degrees of two possible actions where one is ethically preferable to the other.

A principle of ethical preference is defined as a disjunctive normal form predicate in terms of lower bounds for duty differentials of a case. 
\begin{align*}
p(a_1,a_2) &\xleftarrow[]{}\\
&\Delta d_1>=v_{1,1} \wedge\dots\wedge \Delta d_n>=v_{n,1}      \\
&\vee\\
&\vdots\\
&\vee\\
&\Delta d_1>=v_{1,m} \wedge\dots\wedge \Delta d_n>=v_{n,m}  
\end{align*}
where $\Delta d_i$ denotes the difference of a corresponding values of duty \textit{i} in actions$\ a_1$ and$\ a_2$ (the actions of the case in question) and$\ v_{i,j}$ denotes the lower bound of duty \textit{i} in disjunct \textit{j} such that$\ p(a_1,a_2)$ returns \textit{true} if action$\ a_1$  is ethically preferable to action$\ a_2$.

Inductive logic programming (ILP) techniques are used to abstract principles from judgments of ethicists on specific two-action cases where a consensus exists as to the ethically relevant features involved, the relative levels of satisfaction or violation of their correlative duties, and the action that is considered ethically preferable. These techniques result in a set of sets of lower bounds for which principle$\ p$ will return \textit{true} for all positive cases presented to it (i.e. where the first action is ethically preferable to the second) and \textit{false} for all negative cases (i.e. where the first action is \textit{not} ethically preferable to the second). That is, for every positive case, there is a clause of the principle that is true for the differential of the actions of the case and, for every negative case, no clause of the principle returns true for the differential of the actions of the case.  The principle is thus complete and consistent with respect to its training cases.  Further, as each set of lower bounds is a \textit{specialization} of the set of minimal lower bounds sufficient to uncover negative cases, each clause of the principle may inductively cover positive cases other than those used in its training.

A general ethical dilemma analyzer, GenEth \cite{DBLP:journals/paladyn/AndersonA18}, has been developed that, through a dialog with ethicists, helps codify ethical principles in any domain. GenEth uses ILP \cite{Džeroski2001} to infer a principle of ethical action preference from cases that is complete and consistent in relation to these cases. 
As cases are presented to the system, duties and ranges of satisfaction/violation values are determined in GenEth through resolution of contradictions that arise, constructing a concrete representation language that makes explicit features, their possible degrees of presence or absence, duties to maximize or minimize them, and their possible degrees of satisfaction or violation.
Ethical preference is determined from differences of satisfaction/violation values of the corresponding duties of two actions of a case. GenEth abstracts a principle of ethical preference $p(a_1,a_2)$ by incrementally raising selected lower bounds (all initially set at their lowest possible value) so that this principle no longer returns true for any negative cases (cases in which $a_2$ is preferable to $a_1$) while still returning true for all positive cases (cases in which $a_1$ is preferable to $a_2$).

To use this principle to determine a VDA's next action, it is necessary to associate each of the VDA's possible actions with a vector of values representing levels of satisfaction or violation of duties that that action exhibits in the current context. The current context is represented as a set of Boolean perceptions whose values are determined from initial input combined with sensor data (such as the fact that it is time to remind a patient that it is time to take a medication or batteries are in need of recharging). These values are provided as input to a decision tree (abstracted from input/output examples provided by the project ethicist) whose output is the duty satisfaction/violation values appropriate for each action given the context defined by the current Boolean perceptions. Given this information, the principle can serve as a comparison function for a sorting routine that orders actions by ethical preference.

The decision making process of a VDA, then, is as follows:  sense the state of the world and abstract it into a set of Boolean perceptions, determine the vectors of duty satisfaction or violation of all actions with respect to this state using the decision tree, and sort the actions in order of ethical preference using the principle such that the first action in the sorted list is the most ethically preferable one.  Clearly, 
%
several kinds of knowledge 
of a VDA are implicit, including the relation between perceptions and actions determined by the decision tree, the ethical consequences of an action (represented by a vector of duty satisfaction or violation values of the action),  disjuncts in the clauses of the principle that are used to order two actions, and the cases from which these disjuncts are abstracted. Since these kinds of knowledge are informal and somewhat implicit, the current version VDA cannot provide explanations about why an action is taken. 

Our current implementation is in the domain of eldercare where a robot is tasked with assisting an elderly person. Its possible actions include:
charge the robot's battery if low until sufficiently charged;
remind the patient that it's time to take a medication according to a doctor's orders, retrieve that medication and bring it to the patient;
engage the patient if the patient has been immobile for a certain period of time;
warn the patient that an overseer will be notified if the patient refuses medication or does not respond to the robot's attempt to engage the patient;
notify an overseer if there has not been a positive response to a previous warning;
return to a seek task position when no tasks are required. 
For further details, readers are referred to \cite{DBLP:conf/aaai/AndersonAB17}.

\subsection{Formal Argumentation}

Formal argumentation or argumentation in AI, is a formalism for representing and reasoning with inconsistent and incomplete information \cite{hofa}. It also provides various ways for explaining why a claim or a decision is made, in terms of justification, dialogue, and dispute trees \cite{DBLP:conf/comma/CyrasST16}. 

Intuitively, an argumentation system consists of a set of arguments and an attack relation over them. Arguments can be constructed from an underlying knowledge base represented by a logical language, while the attack relation can be defined in terms of the inconsistency of the underlying knowledge.  There are different formalisms for modeling formal argumentation, such as Defeasible Logic Programming (DeLP) \cite{DBLP:journals/tplp/GarciaS04}, APSIC$^+$ \cite{DBLP:journals/argcom/ModgilP14}, Assumption-based Argumentation (ABA) \cite{DBLP:journals/argcom/Toni14}, and Classical Logic-based Argumentation \cite{DBLP:journals/argcom/BesnardH14}. In this paper, the acceptance of an ethical consequence specified by a vector of duty satisfaction/violation can be viewed as an assumption, while the relations between accepting an ethical consequence and an action, and between accepting different ethical consequences with respect to a principle, can be represented by deductive rules. Furthermore, default assumptions in epistemic reasoning can also be represented by deductive rules with assumptions in their premises. Under these considerations, we may adopt ABA as a formalism for representation.
Now, let us first introduce some notions of ABA under the setting of this paper. 

According to \cite{DBLP:journals/argcom/Toni14}, an ABA framework is a tuple $\langle \mathcal{L}, \mathcal{R}, \mathcal{A}, ^{\relbar} \rangle$ where 
\begin{itemize}
\item $\langle \mathcal{L}, \mathcal{R} \rangle$ is a deductive system, with $\mathcal{L}$ the language, and $\mathcal{R}$ a set of rules of the form $\sigma_0 \leftarrow \sigma_1, \dots, \sigma_m$ ($m\ge 0$) with $\sigma_i\in \mathcal{L}$ ($i= 0, \dots, m$);
\item $\mathcal{A} \subseteq \mathcal{L}$ is a (non-empty) set, referred to as assumptions; 
\item $^\relbar$ is a total mapping from  $\mathcal{A}$ into $\mathcal{L}$; $\overline{a}$ is referred to as the contrary of $a$. 
\end{itemize}

Given an ABA framework, arguments can be defined by the notion of \emph{deduction}. In terms of  \cite{DBLP:journals/argcom/Toni14}, a deduction for $\sigma\in \mathcal{L}$ supported by $T\subseteq \mathcal{L}$ and $R\subseteq \mathcal{R}$, denoted $T \vdash^R \sigma$, is a (finite) tree with nodes labelled by sentences in $\mathcal{L}$ or by $\tau$ (when the premise of a rule applied in the tree is empty), the root labelled by $\sigma$, leaves either $\tau$ or sentences in $T$, non-leaves $\sigma^\prime$ with, as children, the elements of the body of some rules in $\mathcal{R}$ with head  $\sigma^\prime$, and $R$ the set of all such rules. When the context is clear, $T \vdash^R \sigma$ is written as $T \vdash \sigma$. Then, an argument for (the claim) $\sigma\in \mathcal{L}$ supported by $A\subseteq \mathcal{A}$ ($A\vdash \sigma$ for short) is a deduction for $\sigma$ supported by $A$ (and some $R\subseteq \mathcal{R}$).

Arguments may attack each other. An argument $T_1 \vdash \sigma_1$ attacks an argument $T_2 \vdash \sigma_2$ if and only if $\sigma_1$ is the contrary of one of the assumptions in $T_2$.

Let $AR$ be a set of arguments constructed from $\langle \mathcal{L}, \mathcal{R}, \mathcal{A}, ^{\relbar} \rangle$, and $ATT \subseteq AR\times AR$ be the attack relation over $AR$. A tuple $(AR, ATT)$ is called an abstract argumentation framework (or AAF in brief). Given an AAF, 
the notion of argumentation semantics in \cite{DBLP:journals/ai/Dung95} can be used to evaluate the status of arguments in $AR$. There are a number of argumentation semantics capturing different intuitions and constraints for evaluating the status of arguments in an AAF, including complete, preferred, grounded and stable, etc. A set of arguments accepted together is called an extension. Various types of extensions under different argumentation semantics can be defined in terms of the notion of admissibility of set of arguments, which is in turn in terms of the notions of conflict-freeness and defense. For $E\subseteq AR$, we say that  $E$ is conflict-free if and only if there exist no $X_1, X_2\in E$  such that $X_1$ attacks $X_2$; $E$ defends an argument $X\in  AR$ if and only if for every argument $Y\in  AR$ if $Y$ attacks $X$ then there exists $Z\in E$ such that $Z$ attacks $Y$. Set $E$ is admissible if and only if it is conflict-free and defends each argument in $E$. Then, we say that:
\begin{itemize}
\item $E$ is a complete extension if and only if $E$ is admissible and each argument in $ AR$ defended by $E$ is in $E$; 
\item $E$ is a preferred extension if and only if $E$ is a maximal complete extension with respect to set-inclusion; 
\item $E$ is the grounded extension if and only if $E$ is a minimal complete extension with respect to set-inclusion; 
\item $E$ is a stable extension if and only if $E$ is conflict-free and for every $X\in AR\setminus E$, there exists $Y\in E$ such that $Y$ attacks $X$. 
\end{itemize}

Given an AAF $(AR, ATT)$, we use $sm(AR, ATT)$ to denote a set of extensions of $(AR, ATT)$ under semantics $sm \in \{\mathrm{Co}, \mathrm{Pr}, \mathrm{Gr}, \mathrm{St}\}$, in which $Co, Pr, Gr$ and $St$ denote complete, preferred, grounded and stable semantics respetively.

It has been verified that each AAF has a unique (possibly empty) set of grounded extension, while many AAFs may have multiple sets of extensions under other semantics. 
 When an AAF is acyclic, it has only one extension under all semantics. 
Then, we say that an argument of an AAF is skeptically justified under a given semantics if it is in every extension of the AAF, and credulously justified if it is in at least one but not all extensions of the AAF. Furthermore, we say that an argument is skeptically (credulously) rejected if it is attacked by a skeptically (respectively, credulously) justified argument. 

\begin{example}[Formal argumentation]
To illustrate the above notions, consider a famous example in nonmonotonic reasoning, called the Nixon diamond, a scenario in which default assumptions lead to mutually inconsistent conclusions:
\begin{itemize}
\item Usually, Quakers are pacifist.
\item Usually, Republicans are not pacifist.
\item Richard Nixon is both a Quaker and a Republican.
\end{itemize}
In terms of ABA, let $\mathcal{L} = \{Quaker(\mathrm{RN}), Republican(\mathrm{RN}), pacifist(\mathrm{RN})$, $\neg pacifist(\mathrm{RN})$, $asm_p(\mathrm{RN})$, $asm_{\neg p}(\mathrm{RN})\}$ where $\mathrm{RN}$ denotes Richard Nixon, $\mathcal{A} = \{asm_p(\mathrm{RN}), asm_{\neg p}(\mathrm{RN})\}$, $\overline{asm_p(\mathrm{RN})} =$ $\neg pacifist(\mathrm{RN})$, $\overline{asm_{\neg p}(\mathrm{RN})}= pacifist(\mathrm{RN})$, and $\mathcal{R} = \{Quaker(\mathrm{RN}) \leftarrow, Republican(\mathrm{RN})\leftarrow, pacifist(\mathrm{RN})\leftarrow Quaker(\mathrm{RN}),asm_p(\mathrm{RN})$, $\neg pacifist(\mathrm{RN})\leftarrow Republican(\mathrm{RN}),asm_{\neg p}(\mathrm{RN})\}$. Then, there are 4 arguments as follows, in which $Y_4$ attacks $Y_1$ and $Y_3$, and $Y_3$ attacks $Y_2$ and $Y_4$, as illustrated in Fig. \ref{fig:ex-111}:
\begin{itemize}
    \item $Y_1:\{asm_p(\mathrm{RN})\}\vdash asm_p(\mathrm{RN})$
    \item $Y_2:\{asm_{\neg p}(\mathrm{RN})\}\vdash asm_{\neg p}(\mathrm{RN})$
    \item $Y_3:\{asm_p(\mathrm{RN})\}\vdash pacifist(\mathrm{RN})$
    \item $Y_4:\{asm_{\neg p}(\mathrm{RN})\}\vdash \neg pacifist(\mathrm{RN})$
\end{itemize}
\begin{figure}[h!]
  \centering
 \includegraphics[width=0.45\textwidth]{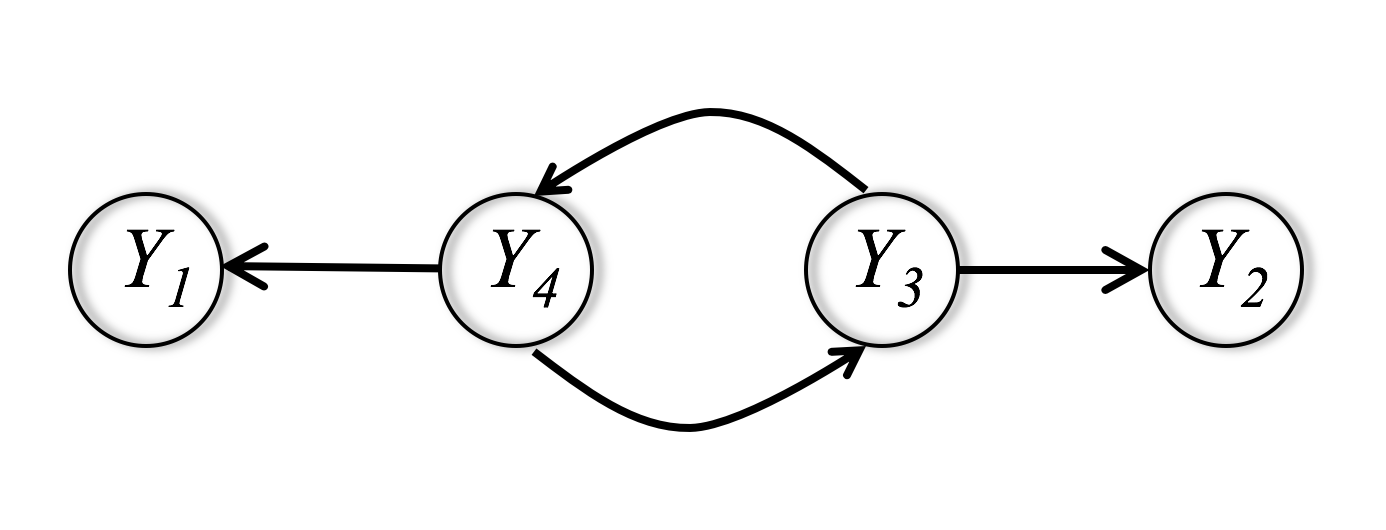}
  \caption{The AAF of the Nixon diamond example.}
    \label{fig:ex-111}
\end{figure}
\end{example}
Then, under grounded semantics, there is only one extension which is an empty set; under complete semantics, there are three extensions $\{Y_1, Y_3\}$, $\{Y_2, Y_4\}$ and $\{\}$; under stable and preferred semantics, there are two extensions  $\{Y_1, Y_3\}$ and $\{Y_2, Y_4\}$. No argument is skeptically justified under all semantics, all arguments are credulously justified under all semantics except grounded. 
For more information about argumentation semantics, please refer to \cite{DBLP:journals/ker/BaroniCG11}.

\section{Representing a value driven agent}
In this section, we introduce a formal language and use it to represent the knowledge and model of a VDA, which lays a foundation for argumentation-based justification and explanation of the decision-making of a VDA. 



The language of a VDA is composed of atoms of perceptions, actions and duties. 


\begin{definition}[Language of a VDA]\label{def-L}
Let $Atom$ be a set of atoms of perceptions, and $Sig$ be a set of signatures. Let $\mathrm{L}= (Atom, A$, $D)$ be a language consisting of 
\begin{itemize}
\item a set of atoms $Atom$, 
\item a set of actions $A\subseteq Sig$, and 
\item a set of duties $D\subseteq Sig$, 
\end{itemize}
such that $A$ and $D$ are disjoint. 
\end{definition}



\begin{example}[Language of a VDA]
In \cite{DBLP:journals/pieee/Anderson19}, there is a set of 10 atoms of perceptions (denoted $Atom_1$): low battery (lb), medication reminder time (mrt), reminded (r), refused medication (rm), fully charged (fc), no interaction (ni), warned (w), persistent immobility (pi), engaged (e), ignored warning (iw);  a set of 6 actions (denoted $A_1$): charge, remind, engage, warn, notify and seek task; and a set of 7 duties (denoted $D_1$): maximize honor commitments ($\mathrm{MHC}$), maximize maintain readiness ($\mathrm{MMR}$), minimize harm to patient ($\mathrm{mH2P}$), maximize good to patient ($\mathrm{MG2P}$), minimize non-interaction ($\mathrm{mNI}$), maximize respect autonomy ($\mathrm{MRA}$) and maximize prevent persistent immobility ($\mathrm{MPPI}$). 
The language of this VDA is then denoted $\mathrm{L}_1 = (Atom_1,  A_1,$ $D_1)$. 
\end{example}

The duties enumerated above have been developed by the project ethicist using GenEth \cite{DBLP:journals/paladyn/AndersonA18} and represent the set needed to drive an eldercare robot in performing the specified actions as shown in \cite{DBLP:journals/pieee/Anderson19}.  

Let $Lit = Atom \cup \{\neg p \mid p\in Atom\}$ be a set of literals. For $l_1, l_2\in L$, we write $l_1 = - l_2$ just in case $l_1 = \neg l_2$ or $l_2 = \neg l_1$. Let $P\subseteq Atom$ be a set of true perceptions. 
Then, the state of the world can be defined in terms of $P$, called a \emph{situation} in this paper, as follows. 


\begin{definition}[Situation] \label{def-situation}
A situation $S$ is a subset of $Lit$,  
such that $S = P\cup \{\neg p \mid p\in Atom\setminus P\}$. 
The set of situations is denoted as $SIT$. 
\end{definition}


\begin{example}[Situation] \label{ex-situation} 
Let $Lit_1 = Atom_1 \cup \{\neg p \mid p\in Atom_1\}$ be a set of literals, which can be extended when a VDA becomes more sophisticated. See Section 4 for details. Let $P_1 = \{mrt, r$, $  rm,  fc\}$ be a set of true perceptions. 
An example of the state of the world: $S_1 = \{\neg lb,   mrt,   r$, $  rm,  fc$, $\neg ni$, $\neg w$, $\neg pi$, $\neg e$, $\neg iw\}$.
\end{example}

Situation $S$ determines the satisfaction and/or violation degree of duties $D$ by actions $A$. In each situation, all duty satisfaction/violation values for each action are determined by a decision tree using the perceptions of the situation as input. A set of vectors of duty satisfaction/violation values of all actions in a situation is called an \emph{action matrix}. The decision tree is derived from a set of known situation/action matrix pairs.


\begin{definition}[Action matrix of a situation]
A duty satisfaction value is a positive integer, while a duty violation value is a negative integer. In addition, if a duty is neither satisfied nor violated by the action, the value is zero. Given an action $\alpha\in A$ and a situation $S\in SIT$, a vector of duty satisfaction/violation values for $\alpha$, denoted as  $v_S(\alpha)$, is a vector $v_S(\alpha) = ( d_1: v_{S, \alpha}(d_1), \dots, d_n: v_{S, \alpha}(d_n))$ where $v_{S, \alpha}(d_i)$ is the satisfaction/violation value of $d_i\in D$ w.r.t $\alpha$ in S. Then, an action matrix of a situation $S$ is defined as $M_S = \{v_S(\alpha) \mid \alpha\in A\}$. The set of action matrices of all situations $SIT$ is denoted as $M_{SIT} = \{M_S \mid S\in SIT\}$. 
\end{definition}

In this definition, a vector of duty satisfaction/violation values represents the \textit{ethical consequences} of its corresponding action in a given situation. An action's ethical consequences are denoted by how much its execution will satisfy or violate each duty. Conflicts arising between actions will be resolved by a principle abstracted from cases.

For brevity, when the order of duties is clear,  $v_S(\alpha) = ( d_1: v_{S, \alpha}(d_1), \dots, d_n: v_{S, \alpha}(d_n)  )$ is also written as $v_S(\alpha) = ( v_{S, \alpha}(d_1), \dots,  v_{S, \alpha}(d_n)  )$.  

\begin{example}[Action matrix of a situation]\label{ex-action-matrix} 
Given a state of the world $S_1$ (as denoted in Example \ref{ex-situation}), derived as described in \cite{DBLP:journals/pieee/Anderson19},  the action matrix of $S_1$ is $M_{S_1} = \{v_{S_1}(charge)$, $v_{S_1}(remind)$, $v_{S_1}(engage)$, $v_{S_1}(warn)$, $v_{S_1}(notify), v_{S_1}(seekTask)\}$, where
\begin{description}
\item $v_{S_1}(charge) = (0, 1,  -1, -1, 0, 0, 0)$,
\item $v_{S_1}(remind) = (-1, -1, -1, -1, 0, 0, 0)$,
\item $v_{S_1}(engage) = (0, -1, -1, -1, 0, 0, 0)$,
\item $v_{S_1}(warn) = (0, 0, 1, -1, 0, -1, 0)$,
\item $v_{S_1}(notify) = (0, 0, 1, -1, 0, -2, 0)$,
\item $v_{S_1}(seekTask) = (0, -1, -1, 1, 0, 0, 0)$.
\end{description}

The duties in each vector are $\mathrm{MHC}$, $\mathrm{MMR}$, $\mathrm{mH2P}$, $\mathrm{MG2P}$, $\mathrm{mNI}$, $\mathrm{MRA}$ and $\mathrm{MPPI}$ in order. Each duty satisfaction/violation vector denotes how much the associated action satisfies or violates each of these duties, positive values representing satisfaction (1=some, 2=much) and negative values representing violation (-1=some, -2=much). The value 0 denotes that an action neither satisfies nor violates a duty. For example, the vector $v_{S_1}(charge)$ specifies that under situation $S_1$, action \textit{charge} satisfies Maximize Maintain Readiness with degree 1, while violating Minimize Harm to Patient and Maximize Good to Patient with degree 1.

For readability, this can also be presented in tabular form:

\begin{tabular}{ |p{1.3cm}||p{1cm}|p{1cm}|p{1cm}|p{1cm}|p{1cm}|p{1cm}|p{1cm}| }
 \hline
 &MHC &MMR &mH2P &MG2P &mNI &MRA &MPPI
 \\
 \hline
 charge &0   &1 &-1 &-1 &0 &0 &0\\
 remind &-1 &-1 &-1 &-1 &0 &0 &0\\
 engage &0 &-1 &-1 &-1 &0 &0 &0\\
 warn   &0 &0 &1 &-1 &0 &-1 &0\\
 notify &0 &0 &1 &-1 &0 &-2 &0\\
 seekTask &0 &-1 &-1 &1 &0 &0 &0\\
 \hline
\end{tabular}
\end{example}

Given a situation and its corresponding action matrix, actions can be sorted in order of ethical preference using a \textit{principle} abstracted from a set of cases by applying ILP techniques.  Clauses of the principle specify learned lower bounds of the differentials between corresponding duties of any two actions that must be met or exceeded to satisfy the clause.

Let $v_{S}(\alpha_1) = ( d_1: v_{S,\alpha_1}(d_1), \dots$, $d_n: v_{S, \alpha_1}(d_n))$ and $v_{S}(\alpha_2) = ( d_1: v_{S,\alpha_2}(d_1)$, $\dots, d_n: v_{S, \alpha_2}(d_n))$ be vectors of duty satisfaction/violation values. In the following definitions, we use 
 $w = v_{S}(\alpha_1) - v_{S}(\alpha_2) = (d_1: w(d_1), \dots, d_n: w(d_n))$ to denote a vector of the differentials of $v_{S}(\alpha_1)$ and $v_{S}(\alpha_2)$, where $w(d_1) = v_{S,\alpha_1}(d_1) - v_{S,\alpha_2}(d_1)$, \dots, $w(d_n) = v_{S,\alpha_1}(d_n) - v_{S,\alpha_2}(d_n)$. 

By considering a set of cases, we may obtain a set of vectors of acceptable lower bounds of satisfaction/violation degree differentials such that all positive cases meet or exceed the lower bounds of some vector, while no negative case does. 

\begin{definition}[Principle]
A principle is defined as  $\pi = \{u_1, \dots, u_k\}$, where $u_i =  ( d_1: u_i(d_1), ..., d_n: u_i(d_n)  )$, where $d_j$ is a duty, and $u_i(d_j)$ is the acceptable lower bound of the differentials between corresponding duties of two actions in $A$. 
\end{definition}

Intuitively, each $u_i$ of a principle is a collection of values denoting how much more an action must, at least, satisfy each duty (or how much, at most, it can violate each duty) than another action for it to be considered the ethically preferable of the pair.  As duties are not necessarily equally weighted nor form a weighted hierarchy, principle $\pi$ is required to determine which duty (or set of duties) is (are) paramount in the current context.  For brevity, when the order of duties is clear, in a principle the lower bounds of the differentials between duties is also written as $u_i =  ( u_i(d_1), ...,  u_i(d_n))$. 

\begin{example}[Principle] \label{ex-principle}
According to \cite{DBLP:journals/pieee/Anderson19}, we have
$\pi_1 = \{u_1, \dots, u_{10}\}$ where 
\begin{description}
\item $u_1 = (-1, -4, -4, -2, -4, -4, 2)$,
\item $u_2 = (-1, -4, -4, -2, 0, 0, 1)$,
\item $u_3 = (0, -3, 0, -1, 0, 1, 0)$,
\item $u_4 = (0, -3, 0, 1, 0, 0, 0)$,
\item $u_5 = (0, -1, 0, 0, 0, 0, 0)$,
\item $u_6 = (0, -3, 0, -1, 1, -1, 0)$,
\item $u_7 = (-1, -4, 1, -2, -4, -4, 0)$,
\item $u_8 = (1, -3, 0, -2, -4, -4, 0)$,
\item $u_9 = (0, 3, 0, -2, 0, 0, 0)$,
\item $u_{10} = (-1, -4, 1, -1, -4, -4, -1)$.
\end{description}
The 10 elements in $\pi_1$ correspond to 10 disjuncts of the principle $p(a_1,a_2)$ in  \cite{DBLP:journals/pieee/Anderson19}. Each disjunct of the principle specifies a relationship between duties of an ordered pair actions that, if held, establishes that the first action of the pair ($a_1$) is ethically preferable to the second ($a_2$). For example, $u_1$ states that action $a_1$ is ethical preferable to action $a_2$ if: $a_2$ satisfies Maximize Honor Commitments no more that 1 more than $a_1$ (or $a_1$ violates it no more that 1 more than $a_2$), $a_2$ satisfies Maximize Good to Patient no more that 2 more than $a_1$ (or $a_1$ violates it no more that 2 more than $a_2$), and $a_1$ satisfies Maximize Prevent Persistent Immobility by at least 2 more than $a_2$ (or $a_2$ violates it by at least 2 more than $a_1$). As the lower bounds of disjunct $u_1$ for each other duty are minimal (i.e. it is not possible given the current ranges of duty satisfaction/violation values to generate a value lower), any relationship between the values of each action is acceptable. 

In tabular form:

\begin{tabular}{ |p{1cm}||p{1cm}|p{1cm}|p{1cm}|p{1cm}|p{1cm}|p{1cm}|p{1cm}| }
 \hline
 &MHC &MMR &mH2P &MG2P &mNI &MRA &MPPI
 \\
 \hline
 $u_1$ &-1 &-4 &-4 &-2 &-4 &-4 &2\\
 $u_2$ &-1 &-4 &-4 &-2 &0 &0 &1\\
 $u_3$ &0 &-3 &0 &-1 &0 &1 &0\\
 $u_4$ &0 &-3 &0 &1 &0 &0 &0\\
 $u_5$ &0 &-1 &0 &0 &0 &0 &0\\
 $u_6$ &0 &-3 &0 &-1 &1 &-1 &0\\
 $u_7$ &-1 &-4 &1 &-2 &-4 &-4 &0\\
 $u_8$ &1 &-3 &0 &-2 &-4 &-4 &0\\
 $u_9$ &0 &3 &0 &-2 &0 &0 &0\\
 $u_{10}$ &-1 &-4 &1 &-1 &-4 &-4 &-1\\
 \hline
\end{tabular}
\end{example}

Given a principle and two vectors of duty satisfaction/violation values, we may define a notion of ethical preference over actions.

 \begin{definition}[Ethical preference over actions]
 Given a principle $\pi$, a situation $S$, and two actions $\alpha_1$ and $\alpha_2$, let $w$ be the differentials of $v_{S}(\alpha_1)$ and $v_{S}(\alpha_2)$ as mentioned above. We say that $\alpha_1$ is ethically preferable (or equal) to $\alpha_2$ with respect to some $u\in \pi$, written as $v_{S}(\alpha_1) \ge_u v_{S}(\alpha_2)$, if and only if  for each $d_i: w(d_i)$ in $w$ and $d_i: u(d_i)$ in $u$, it holds that $w(d_i) \ge  u(d_i)$.
\end{definition}



In this definition, we make explicit the disjuncts $(u)$ in the clause of the principle that are used to order two actions. 

Given two actions $\alpha_1$ and $\alpha_2$, there might exist two different clauses of $\pi$, say $u_1, u_2\in \pi$, such that $v_{S}(\alpha_1) \ge_u v_{S}(\alpha_2)$ and $v_{S}(\alpha_2) \ge_{u^\prime} v_{S}(\alpha_1)$ where $u,u^\prime\in \pi$ and $u\neq u^\prime$. In this case, we say that neither action $\alpha_1$ nor action $\alpha_2$ is ethically preferable to the other.  In other words, according to the principle, there is no ethical justification to choose one over the other.

Based on the above notions, a value driven agent (VDA) is formally defined as follows.

\begin{definition}[Value driven agent]
A value driven agent is a tuple $Ag = (\mathrm{L}, SIT$, $M_{SIT}, \pi)$ where $\mathrm{L}= (Atom$, $A, D)$.  
\end{definition}

\begin{example}[Value driven agent]
According to the above examples, we have $Ag_1 = (\mathrm{L}_1, SIT_1$, $M_{SIT_1}, \pi_1)$ where $SIT_1$ contains $S_1$ and $M_{SIT_1}$ contains $M_{S_1}$.   
\end{example}

In a VDA, given a situation and an action matrix, a set of solutions can be defined as follows.   

\begin{definition}[Solution] \label{def-solution}
Let $Ag = (\mathrm{L}, SIT, M_{SIT}, \pi)$ be a value driven agent, where $\mathrm{L}= (Atom, A, D)$. Given a situation $S\in SIT$ and an action matrix $M_S\in M_{SIT}$, a solution of $Ag$ with respect to $S$ is $\alpha \in A$ if and only if there is an ordering of $M_S$ with respect to $\pi$ such that $v_{S}(\alpha)$ is the first in that ordering.
The set of all solutions of $Ag$ with respect to $S$ is denoted as $sol(Ag, M_S, \pi) = \{\alpha \in A \mid  \alpha$ is a solution of $Ag$ w.r.t. $S\}$.
\end{definition}

\begin{example}[Solution]  
Given $Ag_1 = (\mathrm{L}_1, SIT_1$, $M_{SIT_1}, \pi_1)$, $S_1$ and $M_{S_1}$,  there is a unique ordering of $M_{S_1}$ with respect to $\pi_1$:  $v_{S_1}(warn) \ge_{u_5} v_{S_1}(notify)\ge_{u_7} v_{S_1}(seekTask) \ge_{u_4} v_{S_1}(charge) \ge_{u_5/u_8} v_{S_1}(engage) \ge_{u_5/u_8} v_{S_1}(remind)$. So, $Ag_1$ has only one solution $warn$.
\end{example}

According to Definition \ref{def-solution}, we directly have the following proposition. 

\begin{proposition}[The number of solutions]
Given $Ag = (\mathrm{L}, SIT, M_{SIT}, \pi)$, a situation $S\in SIT$ and an action matrix $M_S\in M_{SIT}$, there are $k$ solutions of $Ag$ if and only if there are $k$ different orderings of
$M_S$ with respect to $\pi$ such that in each ordering the first element is different from the ones in all other orderings. 
\end{proposition}

In summarizing this section, we may conclude that the formal model of a VDA  properly captures the underlying knowledge of a VDA, and, to our knowledge, is the first such formalization.  It lays the foundation for developing a methodology for justifying and explaining the decision-making of a VDA. 

\section{Argumentation-based justification and explanation}
\subsection{ABA-based argumentation systems of a VDA}
As described in the previous section, in a VDA, a decision is made by checking whether there is an ordering over the set of actions according to the ethical preference relations. However,  it is not clear how the ethical consequences of various actions affect each other,  nor how the disjuncts of the principle determine the ordering of ethical consequences of actions. In this paper, we exploit Assumption-Based Argumentation (ABA) for the justification and explanation of the decision-making of a VDA by considering these factors. Furthermore, we go beyond the existing version of VDA, considering not only practical reasoning, but also epistemic reasoning, such that the inconsistency of knowledge can be identified and properly handled. 

With the above considerations in mind, an ABA-based argumentation system for practical reasoning of a VDA under a situation $S$ is defined as follows. 

\begin{definition}[ABA-based argumentation system for practical reasoning of a VDA] \label{def-ABA4VDA}
Let $Ag = (\mathrm{L}$, $SIT$, $M_{SIT}$, $\pi)$ be a value driven agent, where $\mathrm{L}= (Atom, A, D)$.  Given a situation $S\in SIT$, the ABA-based argumentation system of $Ag$ for practical reasoning is denoted as $\langle \mathcal{L}_{Ag, S}, \mathcal{R}_{Ag, S}, \mathcal{A}_{Ag, S}, ^{\relbar} \rangle$, where 
\begin{itemize}
\item $\mathcal{L}_{Ag, S} = \pi\cup M_S \cup\{\neg \phi\mid \phi\in M_S \}\cup A $;
\item each element in $\mathcal{R}_{Ag, S}$ belongs to one of the following types of rules:
\begin{itemize}
\item \emph{action rules} of the form $\alpha \leftarrow v_S(\alpha)$,  where $\alpha\in A$ is an action, $v_S(\alpha)\in M_S$ is a vector of the duty satisfaction/violation values of $\alpha$ in the situation $S$ 
such that there exists $d_i: v_{S,\alpha}(d_i)$ such that $v_{S,\alpha}(d_i)\ge 1$;
\item \emph{principle rules} of the form $\neg v_S(\beta)\leftarrow u, v_S(\alpha)$, such that $v_S(\alpha) \ge_u v_S(\beta)$ where $u\in \pi$, and $\alpha, \beta\in A$;

\end{itemize}
 \item $ \mathcal{A}_{Ag, S} \subseteq  M_S$;  
\item  $^\relbar$ is a total mapping from  $\mathcal{A}_{Ag, S}$ into $\mathcal{L}_{Ag, S}$, such that 
  for all $v_S(\alpha)\in \mathcal{A}_{Ag, S}$,  $\overline{v_S(\alpha)} =_{def} \neg v_S(\alpha)$. 
\end{itemize}
\end{definition}

The items in Definition \ref{def-ABA4VDA} are explained as follows.  

First, the language of argumentation for practical reasoning is composed of the set of disjuncts of a principle, the set of duty satisfaction/violation vectors under a given situation, and the set of actions. 
Among them, the first two sets of elements are assumptions, in the sense that both the disjuncts of a principle and the ethical consequences of actions can be attacked. For all $v_S(\alpha)\in M_S$, we may view $v_S(\alpha)$ as a proposition, meaning that the ethical consequence of $\alpha$ in situation $S$ is acceptable. We use $\neg v_S(\alpha)$ to indicate that it is not the case that $ v_S(\alpha)$ holds.   


Second, there are two types of rules for practical reasoning of a VDA. An action rule $\alpha \leftarrow v_S(\alpha)$ can be understood as: if the ethical consequence of action $\alpha$ (i.e. how it would satisfy and/or violate duties) w.r.t. the decision tree, i.e., $v_S(\alpha)$, is acceptable (i.e. satisfies \textit{some} duty),  then $\alpha$ should be executed.  The acceptance of $v_S(\alpha)$ is an assumption, since there might be some other ethical consequences that are more acceptable  according to the principle $\pi$, in the sense that $v_S(\beta)\ge_u v_S(\alpha)$ for some $\beta\in A$ and $u\in\pi$.
Action rules can be automatically and dynamically generated and updated according to the data from a VDA. 


\begin{example}[Action rule] \label{ex-1}
Continuing Example \ref{ex-action-matrix}. Given 
$S_1$, there are four action rules. 
\begin{description}
\item $r_1: charge \leftarrow v_{S_1}(charge)$. 
\item $r_2: warn \leftarrow v_{S_1}(warn)$. 
\item $r_3: notify \leftarrow v_{S_1}(notify)$.
\item $r_4: seekTask \leftarrow v_{S_1}(SeekTask)$.
 \end{description}
\end{example}

In general, action rules are constructed only for those actions that satisfy at least one duty as those that do not are \textit{a priori} less ethically preferable. In the example, neither \textit{remind} nor \textit{engage} satisfy any duty and, thus, no action rule is generated for either. Theoretically, it is possible that \textit{no} action satisfies \textit{any} duty in a given situation.  In that case, the most preferable action would be among those that violated duties the least so action rules for all actions would be constructed. 

Principle rules can be constructed in terms of the priority relation between two vectors of duty satisfaction/violation duties with respect to a principle. For all $\alpha, \beta\in A$, if $v_S(\alpha) \ge_u v_S(\beta)$, then we have a principle rule $\neg v_{S}(\beta) \leftarrow u, v_{S}(\alpha)$, indicating that if both $u$ and $v_{S}(\alpha)$ are accepted, then it is not the case that $v_{S}(\beta)$ is acceptable.  

\begin{example}[Principle rules] \label{ex-pr-1}
Continuing Examples \ref{ex-action-matrix} and \ref{ex-principle}. Given 
$S_1$, there are six principle rules. 
\begin{description}
\item $r_5: \neg v_{S_1}(charge) \leftarrow u_7, v_{S_1}(warn)$. 
\item $r_6: \neg v_{S_1}(charge) \leftarrow u_7, v_{S_1}(notify)$. 
\item $r_7: \neg v_{S_1}(charge) \leftarrow u_4, v_{S_1}(seekTask)$. 
\item $r_8:  \neg v_{S_1}(notify) \leftarrow  u_5, v_{S_1}(warn)$.
\item $r_9: \neg v_{S_1}(seekTask) \leftarrow  u_7, v_{S_1}(warn)$. 
\item $r_{10}: \neg v_{S_1}(seekTask) \leftarrow u_7, v_{S_1}(notify)$.  
 \end{description}
\end{example}


Third, regarding practical reasoning, for simplicity, we assume that only the elements in $M_S$ may be assumptions. 
\begin{example}[Assumptions]\label{ex-ep-assumps}
Continue Example \ref{ex-action-matrix}. For each duty satisfaction/violation vector, if at least one duty in the vector is satisfied, then the vector is regarded as an assumption. So, for practical reasoning of $Ag_1$ under situation $S_1$, we have a set of assumptions denoted as  $\mathcal{A}_{Ag_1, S_1} = \{v_{S_1}(charge), v_{S_1}(warn), v_{S_1}(notify), v_{S_1}(SeekTask)\}$. 
\end{example} 

Fourth, concerning the contrary of each element in $\mathcal{A}_{Ag,S}$,  for each vector $v_S(\alpha)$ of duty satisfaction/violation, its contrary is its negation. 

%

In Definitions \ref{def-ABA4VDA} and \ref{def-situation},  we assume that situation $S$ is given, and can be defined directly by a set of perceptions without epistemic reasoning. However, in many cases, perceptions are unreliable, and a VDA usually only has incomplete and uncertain information. To properly capture the state of the world based on a set of perceptions, epistemic reasoning is needed for inferring implicit knowledge about the world, and for handling inconsistency of knowledge of a VDA. Corresponding to practical reasoning, an ABA-based argumentation system for epistemic reasoning is  defined as follows. 

\begin{definition}[ABA-based argumentation system for epistemic reasoning of a VDA] \label{def-ep-arg-system}
Let $\mathrm{L}= (Atom, A, D)$ be the language of a VDA. Let $Lit = Atom \cup \{\neg p \mid p\in Atom\}$ be the set of literals of the VDA. The ABA-based argumentation system for epistemic reasoning of the VDA is denoted as $\langle {Lit}, \mathcal{R}_{Lit}, \mathcal{A}_{Lit}, ^{\relbar} \rangle$, where 
\begin{itemize}
\item each element in $\mathcal{R}_{Lit}$ is an \emph{epistemic rule} of the form $p \leftarrow p_1, \dots, p_n$ where $p, p_i\in {Lit}$;
 \item  $ \mathcal{A}_{Lit} \subseteq  {Lit}$; 
\item  $^\relbar$ is a total mapping from  $\mathcal{A}_{Lit}$ to ${Lit}$, such that 
 for all $p\in \mathcal{A}_{Lit}$, $\overline{p}\in {Lit}$.
 
\end{itemize}
\end{definition}

In this definition, epistemic rules are used to reason about the state of the world. Consider the following example.

\begin{example}[Epistemic rules]\label{ex-ep-rules}
In situation $S_2$, let the set of true perceptions be $P_2 = \{mrt, r, rm, fc, lb, ab\}$, where $ab$ is a new atom being added to $Atom_1$, denoting that the battery is abnormal. Let $Atom_2 = Atom_1 \cup \{ab\}$. In terms Definition \ref{def-situation}, $S_2 = P_2\cup \{\neg p \mid p\in Atom_2\setminus P_2\}$ which is inconsistent if $lb$ (`low battery') and $fc$ (`fully charged') cannot hold at the same time. From the perspective of epistemic reasoning, some of perceptions  can be viewed as assumptions, e.g., `low battery', `fully charged'. Meanwhile, due to incomplete information, assume that `battery is not abnormal'. In addition, it is reasonable that if  `low battery'  holds  then `fully charged' does not hold, and if `fully charged' holds and the battery is normal then  `low battery' does not hold. 
Under this setting, we have three epistemic rules for reasoning about assumptions: $r_{11}: \neg fc \leftarrow lb$, $r_{12}: \neg lb \leftarrow fc, \neg ab$,  
and $r_{13}: ab \leftarrow $. In addition, there are other epistemic rules corresponding to the facts, including $mrt$, $r$, and $rm$, etc. Since they have no interactions with assumptions and other rules, for simplicity, they are omitted.  Let $Lit_2 = Lit_1\cup \{ab, \neg ab\}$, and $\mathcal{R}_{Lit_2} = \{r_{11}, r_{12}, r_{13}\}$.  

\end{example} 


Then, given an ABA-based argumentation system of a VDA under a situation $S$, arguments and attacks can be defined as follows. 

\begin{definition}[Arguments and attacks]
Let $Ag = (\mathrm{L}$, $SIT, M_{SIT}, \pi)$ be a value driven agent, where $\mathrm{L}= (Atom$, $A, D)$, $\langle \mathcal{L}_{Ag, S}, \mathcal{R}_{Ag, S}, \mathcal{A}_{Ag, S}, ^{\relbar} \rangle$ be an ABA-based argumentation system for practical reasoning of $Ag$ under a situation $S$, and $\langle {Lit}, \mathcal{R}_{Lit}$, $\mathcal{A}_{Lit}, ^{\relbar} \rangle$ be an ABA-based argumentation system for epistemic reasoning of $Ag$.
An argument for $\sigma\in \mathcal{L}_{Ag, S}$ supported by $T\subseteq \mathcal{A}_{Ag, S}$ ( respectively for $\sigma\in \mathcal{L}_{Lit}$ supported by $T\subseteq \mathcal{A}_{Lit}$), written as $T\vdash \sigma$, is a deduction for $\sigma$ supported by $T$.  The conclusion of $T\vdash \sigma$, denoted $concl(T\vdash \sigma)$, is $\sigma$.
An argument $A_1\vdash \sigma_1$ attacks an argument $A_2\vdash \sigma_2$ iff $\sigma_1$ is the contrary of one of the assumptions in $A_2$.  
\end{definition}

The set of arguments constructed from $\langle \mathcal{L}_{Ag, S}, \mathcal{R}_{Ag, S}, \mathcal{A}_{Ag, S}, ^{\relbar} \rangle$ is denoted as $AR_{Ag, S}$, and the set of attacks between the arguments in $AR_{Ag, S}$ is denoted as $ATT_{Ag, S}$. In terms of \cite{DBLP:journals/ai/Dung95}, we call $(AR_{Ag, S}, ATT_{Ag, S})$ an abstract argumentation framework (or AAF for short). Respectively, the AAF constructed from $\langle \mathcal{L}_{Lit}, \mathcal{R}_{Lit}$, $\mathcal{A}_{Lit}, ^{\relbar} \rangle$ is denoted as $(AR_{Lit}, ATT_{Lit})$.

In the remaining part of this section, let us further illustrate the AAFs for practical reasoning and epistemic reasoning of a VDA. On one hand, the AAF constructed from $\langle \mathcal{L}_{Ag, S}, \mathcal{R}_{Ag, S}, \mathcal{A}_{Ag, S}, ^{\relbar} \rangle$ is about practical reasoning (i.e., selecting ethically preferable actions), corresponding to the current version of the VDA in  \cite{DBLP:journals/pieee/Anderson19}. Under this setting, all perceptions are assumed to be facts. No justification about perceptions is considered.   

\begin{example}[AAF for the practical reasoning of a VDA under a given situation]  \label{ex-2}
Continuing Examples~\ref{ex-1} and \ref{ex-pr-1}. Given $\langle \mathcal{L}_{Ag_1, S_1}, \mathcal{R}_{Ag_1, S_1}$, $\mathcal{A}_{Ag_1, S_1}, ^{\relbar} \rangle$ where $\mathcal{L}_{Ag_1, S_1} = \pi_1\cup M_{S_1}\cup \{\neg \phi\mid\phi\in M_{S_1}\}\cup A_1$,  $\mathcal{R}_{Ag_1, S_1} = \{r_1,\dots, r_{10}\}$, and $\mathcal{A}_{Ag_1, S_1} = \{v_{S_1}(charge), v_{S_1}(warn), v_{S_1}(notify), v_{S_1}(SeekTask)\}$, 
we have the following 10 arguments. 
Attacks between arguments are visualized in Fig. \ref{fig:ex-2}.
\begin{itemize}
\item[]$X_{1}$:  $\{v_{S_1}(charge)\} \vdash charge$
\item[]$X_{2}$ : $\{v_{S_1}(warn)\} \vdash warn$
\item[]$X_{3}$ : $\{v_{S_1}(notify)\} \vdash notify$
\item[]$X_{4}$ : $\{v_{S_1}(seekTask)\} \vdash seekTask$
\item[]$X_{5}$ : $\{u_7, v_{S_1}(warn)\} \vdash \neg v_{S_1}(charge)$
\item[]$X_{6}$:  $\{u_7, v_{S_1}(notify)\} \vdash \neg v_{S_1}(charge)$
\item[]$X_{7}$ : $\{u_4, v_{S_1}(seekTask)\} \vdash \neg v_{S_1}(charge)$
\item[]$X_{8}$ : $\{u_5, v_{S_1}(warn)\} \vdash \neg v_{S_1}(notify)$
\item[]$X_{9}$ : $\{u_7, v_{S_1}(warn))\} \vdash \neg v_{S_1}(seekTask)$
\item[]$X_{10}$ : $\{u_7, v_{S_1}(notify)\} \vdash \neg v_{S_1}(seekTask)$
\end{itemize} 

\begin{figure}[h!]
  \centering
 \includegraphics[width=0.55\textwidth]{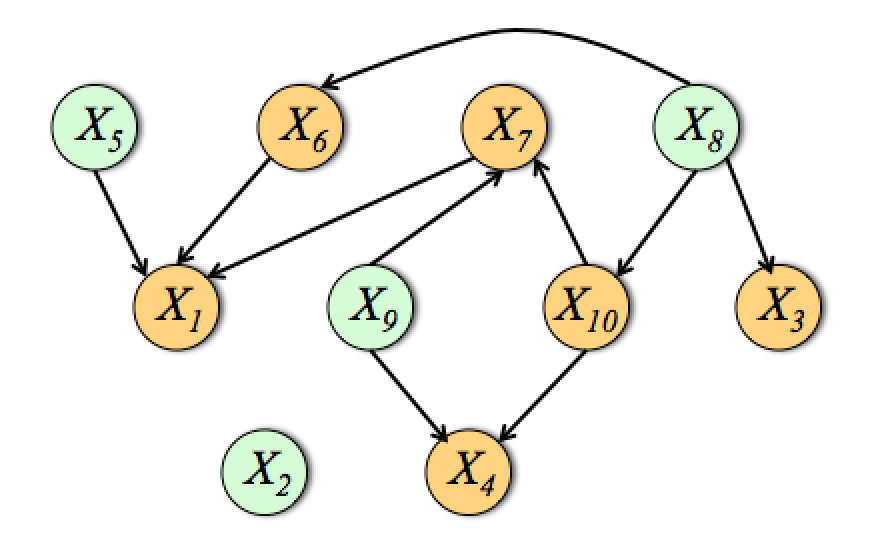}
  \caption{An example of an AAF in situation $S_1$ without considering epistemic reasoning.}
    \label{fig:ex-2}
\end{figure}

On the other hand, as mentioned in Example \ref{ex-ep-rules}, some of perceptions are assumptions that can be in conflict:
when `low battery' and `fully charged' both hold according to the observations,  there exists a conflict between them.  This conflict cannot be identified when only practical reasoning is  considered \cite{DBLP:journals/pieee/Anderson19}. The following example introduces an AAF that can be used for identifying the state of the world of a VDA based on handling the conflicts of its knowledge, i.e., the set of perceptions and epistemic rules. 



\begin{example}[AAF for the epistemic reasoning of a VDA]
In situation $S_2$, given  $\mathcal{R}_{Lit_2} = \{r_{11}, r_{12}, r_{13}\}$ and $\mathcal{A}_{Lit_2} = \{fc, lb, \neg ab\}$, 
there are the following six arguments for epistemic reasoning.  Attacks between arguments are visualized in Fig.\ref{fig:ex-31}.
\end{example}

\begin{itemize}
\item[]$Y_{1}$ : $\{fc\} \vdash fc$
\item[]$Y_{2}$ : $\{lb\} \vdash lb$
\item[]$Y_{3}$ : $\{\neg ab\} \vdash \neg ab$
\item[]$Y_{4}$ : $\{lb\} \vdash \neg fc$
\item[]$Y_{5}$ : $\{fc, \neg ab\} \vdash \neg lb$
\item[]$Y_{6}$ : $\{ \} \vdash ab$
\end{itemize} 

\begin{figure}[h!]
  \centering
 \includegraphics[width=0.4\textwidth]{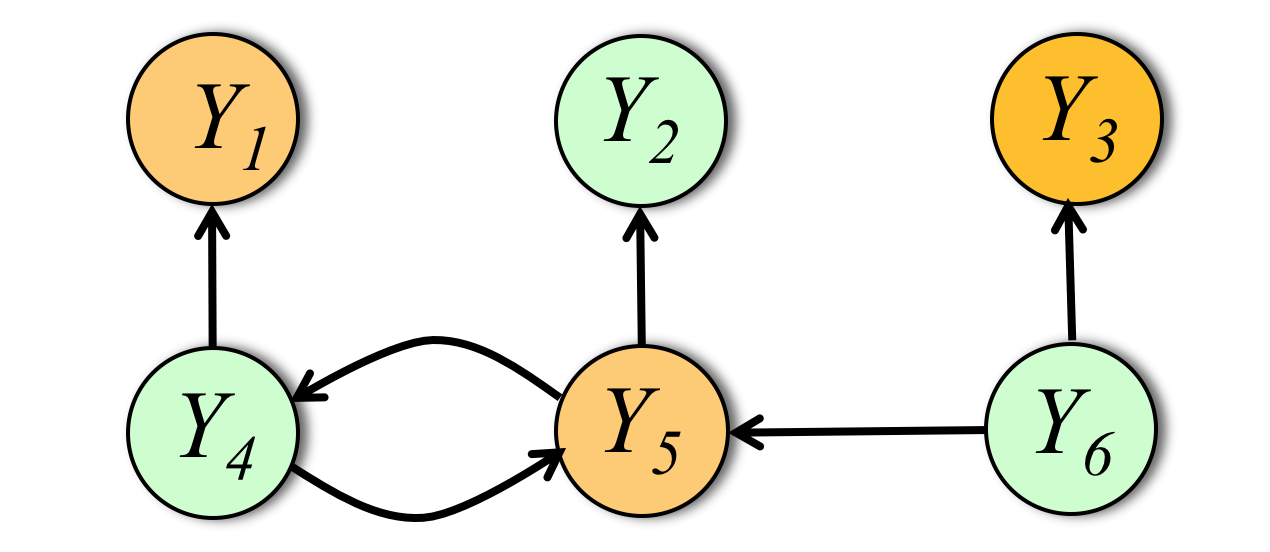}
  \caption{An example of an AAF for epistemic reasoning.}
    \label{fig:ex-31}
\end{figure}
\end{example}

Note that since the situation $S_2$ is determined by the status of the arguments for epistemic reasoning (in this example, arguments $Y_{1}, \dots, Y_{6}$), arguments for practical reasoning might change accordingly. We will further discuss this issue in next subsection. 

\subsection{Argumentation-based justification}
In this subsection, we introduce an argumentation-based approach for justifying an action and a situation of a VDA. Firstly, corresponding to the existing version of the VDA in \cite{DBLP:journals/pieee/Anderson19}, justification is only about actions.

\begin{example}[Extensions of an AAF for practical reasoning] \label{ex-argsem}
Consider the AAF in Figure \ref{fig:ex-2}. It is acyclic and has only one extension under any argumentation semantics, i.e., $E_1 = \{X_2, X_{5}, X_{8}, X_{9}\}$. In this example, all arguments in $E$ are skeptically justified.
\end{example}

\begin{proposition}[Unique complete extension] \label{prop-unique}
Let $Ag = (\mathrm{L}$, $SIT$, $M_{SIT}, \pi)$ be a value driven agent, $\langle \mathcal{L}_{Ag, S}$, $\mathcal{R}_{Ag, S}, \mathcal{A}_{Ag, S}, ^{\relbar} \rangle$ an ABA-based argumentation system for practical reasoning of $Ag$ under a situation $S$, and $(AR_{Ag, S}, ATT_{Ag, S})$ the AAF constructed from the argumentation system. If $Ag$ has a unique solution with respect to $S$, then $(AR_{Ag, S}$, $ATT_{Ag, S})$ has a unique complete extension, which coincides with the unique grounded extension.
\end{proposition}

\begin{proof}
Let $\alpha\in A$ be the unique solution of $Ag$. Since there is no other ordering such that a different action (other than $\alpha$) can be the first action in the sorted list, for all $\beta\in A\setminus \{\alpha\}$, it holds that $v_S(\alpha) >_u v_S(\beta)$ for some $u\in \pi$. This means that each argument $v_S(\beta) \vdash \beta$ is attacked by $v_S(\alpha), u \vdash \neg v_S(\beta)$, which has no attacker. Let $E\subseteq AR_{Ag, S}$ be the set containing $v_S(\alpha) \vdash \alpha$ and all arguments of the form $v_S(\alpha), u \vdash \neg v_S(\beta)$ for all $\beta\in A\setminus \{\alpha\}$.  It turns out that $E$ is the unique complete extension of $(AR_{Ag, S}, ATT_{Ag, S})$.
\end{proof}

Given a set of justified arguments, we may define the set of justified conclusions as follows. 

\begin{definition}[Justified conclusion in practical reasoning]
Let $(AR_{Ag, S}, ATT_{Ag, S})$ be an AAF for practical reasoning, and $X\in AR_{Ag, S}$ be a skeptically (credulously) justified argument under a given argumentation semantics. A skeptically (credulously) justified conclusion is written as $concl(X)$. We say that $concl(X)$ is a skeptically (credulously) justified action if and only if the conclusion of $X$ is an action.
\end{definition}

\begin{example}[Justified conclusions  in practical reasoning] \label{ex-just}
According to Example \ref{ex-argsem}, all elements in $E_1$ are justified conclusions, in which $warn$ is a skeptically justified action since $X_2$ is skeptically justified and its conclusion is an action.
\end{example}

Now, let us verify that the representation by using argumentation-based approach is sound and complete under all semantics mentioned above (i.e., complete, grounded, preferred, stable), in the sense that when a VDA has multiple solutions, each solution of the VDA corresponds exactly to a credulously justified action of the argumentation framework; and when a VDA has a unique solution, the solution corresponds to the unique skeptically justified action of the argumentation framework. 

\begin{proposition}[Soundness and completeness of representation]
Let $Ag = (\mathrm{L}$, $SIT$, $M_{SIT}, \pi)$ be a value driven agent and $\langle \mathcal{L}_{Ag, S}, \mathcal{R}_{Ag, S}, \mathcal{A}_{Ag, S}, ^{\relbar} \rangle$ an ABA-based argumentation system for practical reasoning of $Ag$ under a situation $S$. For all $\alpha\in A$, it holds that: $\alpha$ is one of the solutions of $Ag$ with respect to $S$, if and only if $\alpha$ is a credulously justified action in $(AR_{Ag, S}, ATT_{Ag, S})$.
\end{proposition}

\begin{proof}
On one hand, if $\alpha$ is a  solution of $Ag$ with respect to $S$, then there exists an ordering over $A$, such that $\alpha$ is the first action of the sorted list. Let $X\in  \mathcal{L}_{Ag, S}$ be the argument with $\alpha$ its conclusion, and of the form $v_S(\alpha)\vdash \alpha$. For every argument $Y\in AR_{Ag, S} \setminus \{X\}$, if $Y$ is of the form $v_S(\beta), u \vdash \neg v_S(\alpha)$ for some $u\in \pi$, then since $\alpha$ is the first action of the sorted list, it holds that $v_S(\alpha) \ge_{u^\prime} v_S(\beta)$, and there exists an argument $X^\prime$ of the form $v_S(\alpha), u^\prime \vdash \neg v_S(\beta)$ for some $u^\prime \in \pi$. Let $E$ be the set containing $X$ and all arguments  of the form $v_S(\alpha), u^\prime \vdash \neg v_S(\beta)$. Since $E$ is conflict free, and each attacker of $X$ and $X^\prime$ (i.e., $Y$) is attacked by $X^\prime$, $E$ is admissible, and therefore $X\in E$ is credulously justified. Since $\alpha$ is the conclusion of $X$, it is a credulously justified action in $(AR_{Ag, S}, ATT_{Ag, S})$.




On the other hand, if $\alpha$ is a credulously justified action in $(AR_{Ag, S}, ATT_{Ag, S})$, then there exist an argument $X\in AR_{Ag, S}$ of the form $v_S(\alpha)\vdash \alpha$ and an admissible set $E \subseteq AR_{Ag, S}$ such that $X\in E$. For every action $\beta\in A$, if $\beta\neq \alpha$, it is not the case that $v_{S}(\beta) >_u v_{S}(\alpha)$ for some $u\in\pi$.  Otherwise, there exists an argument of the form $v_S(\beta), u  \vdash \neg v_S(\alpha)$. As a result, $X$ is not in $E$. Contradiction. As a result, we may construct an ordering of actions such that $\alpha$ is the first action in the sorted list. Therefore, $\alpha$ is a solution of $Ag$ with respect to $S$.
\end{proof}

When $\alpha$ is a unique solution of $Ag$ with respect to $S$, according to Proposition \ref{prop-unique}, $(AR_{Ag, S}, ATT_{Ag, S})$ has a unique complete extension. Therefore, $\alpha$ is a unique skeptically justified action in $(AR_{Ag, S}, ATT_{Ag, S})$.

Secondly, to justify a situation, we use an AAF for epistemic reasoning. 

\begin{example}[Extensions of an AAF for epistemic reasoning] \label{ex-argsem-epistemic}
Consider the AAF in Figure \ref{fig:ex-31}. It is also acyclic and has only one extension under any argumentation semantics, i.e., $E_2 = \{Y_{2}, Y_{4}, Y_{6}\}$. The set of conclusions of arguments in $E_2$ is denoted as $concl(E_2) = \{lb, \neg fc, ab\}$. 
\end{example}

\begin{definition}[Skeptically justified/rejected assumption in epistemic reasoning]
Let $(AR_{Lit}, ATT_{Lit})$ be an AAF for epistemic reasoning, and $X\in AR_{Lit}$ be a skeptically justified/rejected argument under a given argumentation semantics. We say that $concl(X)$ is a skeptically justified/rejected assumption if and only if the conclusion of $X$ is an assumption.
\end{definition}

If every assumption is either skeptically justified or skeptically rejected, then a situation containing all justified assumptions is skeptical justified. 

\begin{definition}[Justified situation]
Given $Atom$ and $Lit$, let $P\subseteq Atom$ be a set of perceptions, $\mathcal{A}_{Lit}\subseteq Lit$ be a set of assumptions, and $(AR_{Lit}, ATT_{Lit})$ be an AAF constructed from $\langle {Lit}, \mathcal{R}_{Lit}, \mathcal{A}_{Lit}, ^{\relbar} \rangle$. Let $\mathcal{A}_{Lit}^J \subseteq \mathcal{A}_{Lit}$ be a set of skeptically justified assumptions. The set of justified perceptions is $P^J = (P\setminus \mathcal{A}_{Lit})\cup \mathcal{A}^J_{Lit}$.  If every assumption in $\mathcal{A}_{Lit}\setminus\mathcal{A}_{Lit}^J$ is skeptically rejected, then there is a skeptically justified situation $S^J = P^J \cup \{\neg p \mid p\in  Atom\setminus P^J\}$.
\end{definition}

\begin{example}
Given $\mathcal{A}_{Lit_2} = \{fc, lb, \neg ab\}$ and $P_2 = \{lb, mrt, r, rm, fc, ab\}$, we have $P^J_2 = \{lb, mrt, r, rm, ab\}$ and $S_2^J = \{lb, mrt, r, rm, ab, \neg fc, \neg ni, \neg w, \neg pi, \neg e, \neg iw\}$.
\end{example}

Given $S_2^J$, $M_{S^J_2}$ is generated dynamically, presented in tabular form as follows. 

\begin{tabular}{ |p{1.3cm}||p{1cm}|p{1cm}|p{1cm}|p{1cm}|p{1cm}|p{1cm}|p{1cm}| }
 \hline
 &MHC &MMR &mH2P &MG2P &mNI &MRA &MPPI
 \\
 \hline
 charge &0 &2 &-1 &-1 &0 &0 &0\\
 remind &-1 &-2 &-1 &-1 &0 &0 &0\\
 engage &0 &-2 &-1 &-1 &0 &0 &0\\
 warn   &0 &0 &1 &-1 &0 &-1 &0\\
 notify &0 &0 &1 &-1 &0 &-2 &0\\
 seekTask &0 &-1 &-1 &1 &0 &0 &0\\
 \hline
\end{tabular}

In situation $S_2^J$, actions rules and principles rules are the same as those in situation $S_1$ except that the subscript $S_1$ is substituted by $S_2^J$.





\subsection{Argumentation-based explanation}
Besides justifying perceptions and actions, argumentation provides a natural way for explaining why an action is selected or not, by using the notion of justification of arguments and their premises.

\begin{definition}[Explanation of a justified action]
Let $(AR_{Ag, S}, ATT_{Ag, S})$ be an AAF, and $X$ be a skeptically (credulously) justified argument of the form $v_S(\alpha) \vdash \alpha$ under a given argumentation semantics. The explanation of $\alpha$ being a skeptically (credulously) justified action is that: the argument $v_S(\alpha) \vdash \alpha$ is in every extension (one of the extensions) of  $(AR_{Ag, S}, ATT_{Ag, S})$, which is in turn because the assumption $v_S(\alpha)$ is accepted since it has no attacker or all its attachers are attacked by an argument in each extension (respectively, the given extension). 
\end{definition}  

\begin{example}[Explanation of a justified action]
According to Examples \ref{ex-2} and \ref{ex-just}, the explanation of  the action $warn$ being a skeptically justified action is as follows.
\begin{itemize}
\item $X_{2} = v_{S_1}(warn) \vdash warn$ is in the unique extension $E_1$, because: 
\item the assumption (ethical consequence) $v_{S_1}(warn) = (0, 0, 1, -1, 0, -1, 0) $ is accepted since it has no attacker.  
\end{itemize}
\begin{figure}[h!]
  \centering
 \includegraphics[width=0.4\textwidth]{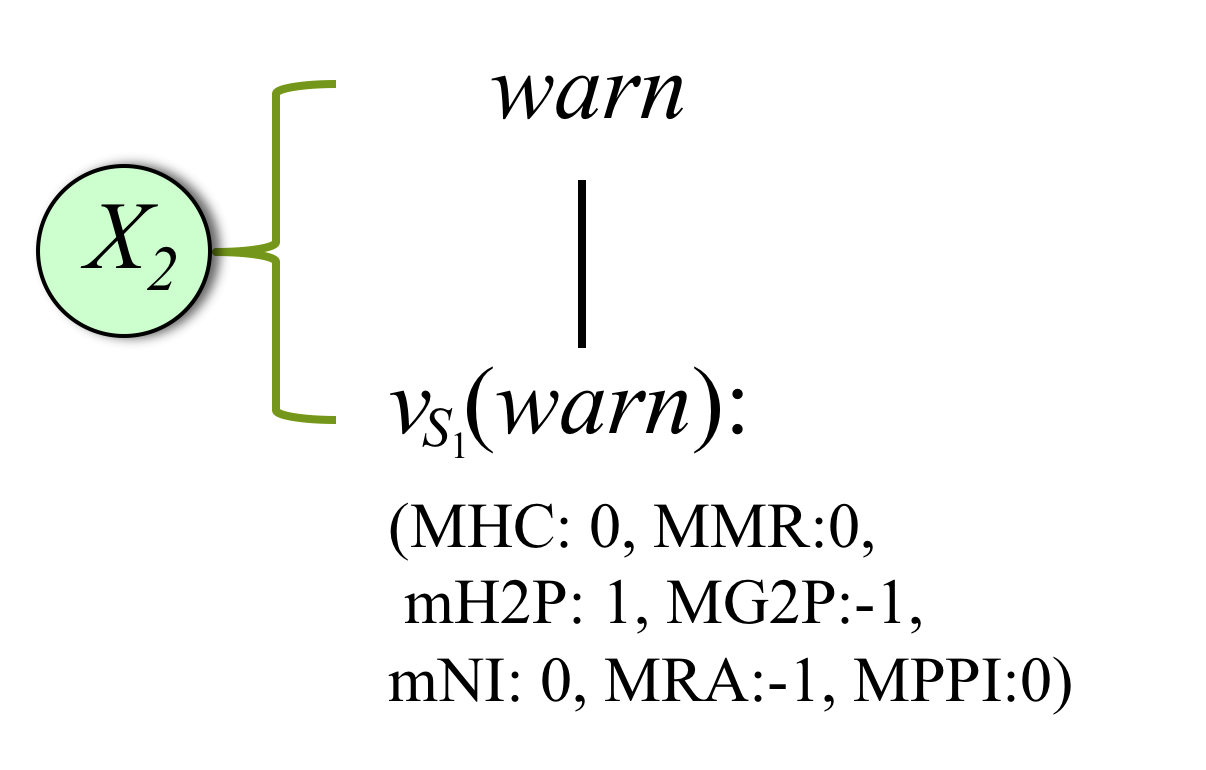}
  \caption{Explanation for a justified action}
    \label{fig:ex-3}
\end{figure}
\end{example}


In this example, accepting the ethical consequence $v_{S_1}(warn)$ means that, under situation $S_1$, $warn$'s satisfaction of Minimize Harm to Patient with degree 1 overrides both  degree 1 violations of Maximize Good to Patient and Maximize Respect Autonomy.

\begin{definition}[Explanation of a rejected action]
Let $(AR_{Ag, S}, ATT_{Ag, S})$ be an AAF, and $X$ of the form $v_S(\alpha) \vdash \alpha$ be a rejected argument under a given argumentation semantics. The explanation of $\alpha$ being a rejected action is that: the argument $v_S(\alpha) \vdash \alpha$ is not in any extension of  $(AR_{Ag, S}, ATT_{Ag, S})$, which is in turn because the assumption (ethical consequence) $v_S(\alpha)$ is not acceptable since in every extension of $(AR_{Ag, S}, ATT_{Ag, S})$, there is an argument attacking $v_S(\alpha) \vdash \alpha$, whose premises (the ethical consequence of another action and a disjunct of the principle) are accepted. 
\end{definition}  

\begin{example}[Explanation of a rejected action]
The explanation of $charge$ being a rejected action is as follows. 
\begin{itemize}
\item The argument $X_{1} = v_{S_1}(charge) \vdash charge$ is not in the unique extension $E_1$, in that:
\item the ethical consequence $v_{S_1}(charge) = (0, 1, -1, -1, 0, 0, 0) $ is not acceptable, since $X_1 = v_{S_1}(charge)\vdash charge$ is attacked by $X_5 = u_7, v_{S_1}(warn) \vdash \neg v_{S_1}(charge)$, whose premises ($u_7$ and $v_{S_1}(warn)$) are accepted.
\end{itemize}
\begin{figure}[h!]
  \centering
 \includegraphics[width=0.8\textwidth]{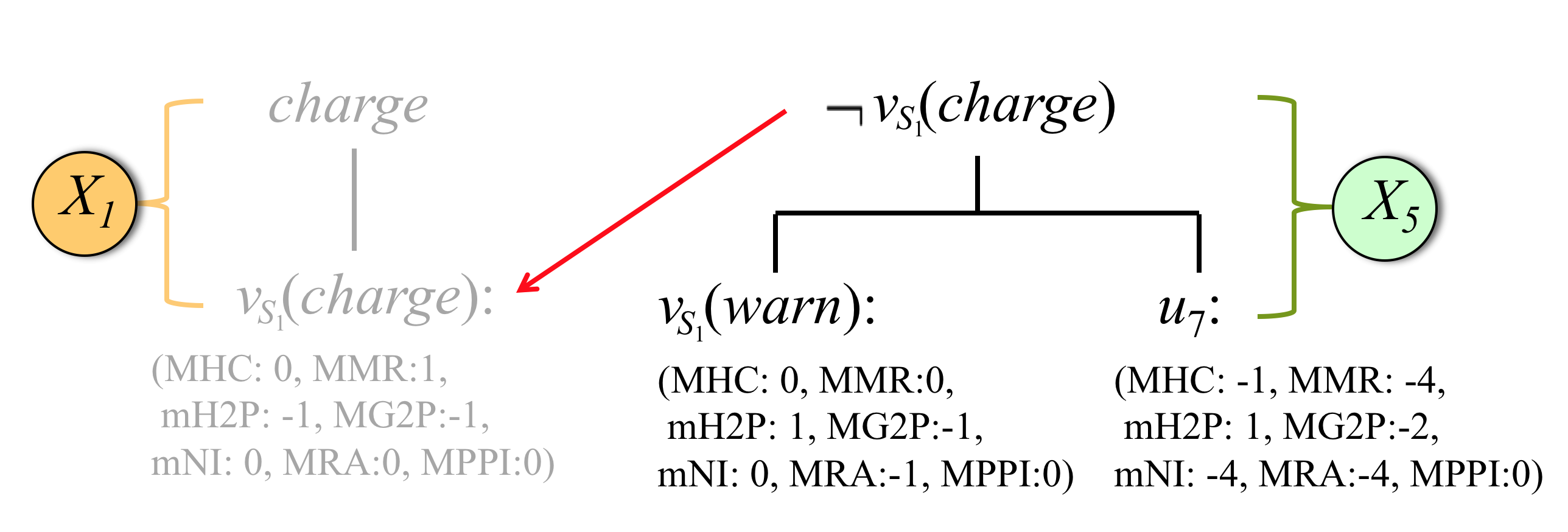}
  \caption{Explanation for a rejected action}
    \label{fig:ex-4}
\end{figure}
\end{example}

In this example, the content of the premises of arguments can be used to further explain the decision. More specifically, given $u_7$ and the ethical consequence depicted by $v_{S_1}(warn)$, the ethical consequence depicted by $v_{S_1}(charge)$ is not acceptable, i.e., $\neg v_{S_1}(charge)$ is acceptable. In other words, under situation $S_1$, satisfying ``minimize harm to patient'' with degree 1 (mH2P:1), even while violating ``maximize respect autonomy'' with degree -1 (MRA: -1), is ethically preferable to satisfying ``maximize maintain readiness'' with degree 1 (MMR: 1). As both actions violate "maximize good to the patient" equally (MG2P: -1), that duty does not help differentiate these actions and therefore has no role in this explanation.

Concerning the explanation of a justified situation, since only assumptions need to be justified, we have the following definition.

\begin{definition}[Explanation of a justified situation]
The explanation of $S$ being a skeptically justified situation is that: each $l \in\mathcal{A}_{Lit}$ is skeptically justified or rejected. In turn, the explanation of a justification/rejection of a literal is similar to that of an action.  
\end{definition}  

\begin{example}
$S_2^J$ is a skeptically justified situation because in the set of assumptions $\{fc, \neg ab, lb\}$, $fc$ and $\neg ab$ are skeptically rejected, and $lb$ is skeptically accepted. The explanation of skeptical justification and rejection of perceptions is in turn described as follows (illustrated in Fig \ref{fig:ex-epistemic}).  
\begin{itemize}
    \item $fc$ and $\neg ab$ are skeptically rejected, because the argument supporting $fc$ (respectively, $\neg ab$) is attacked by a skeptically accepted argument. 
    \item $lb$ is skeptically justified because the argument supporting $lb$ is defended by two skeptically accepted arguments $Y_4$ and $Y_6$.
\end{itemize}
\end{example}

\begin{figure}[h!]
  \centering
 \includegraphics[width=0.8\textwidth]{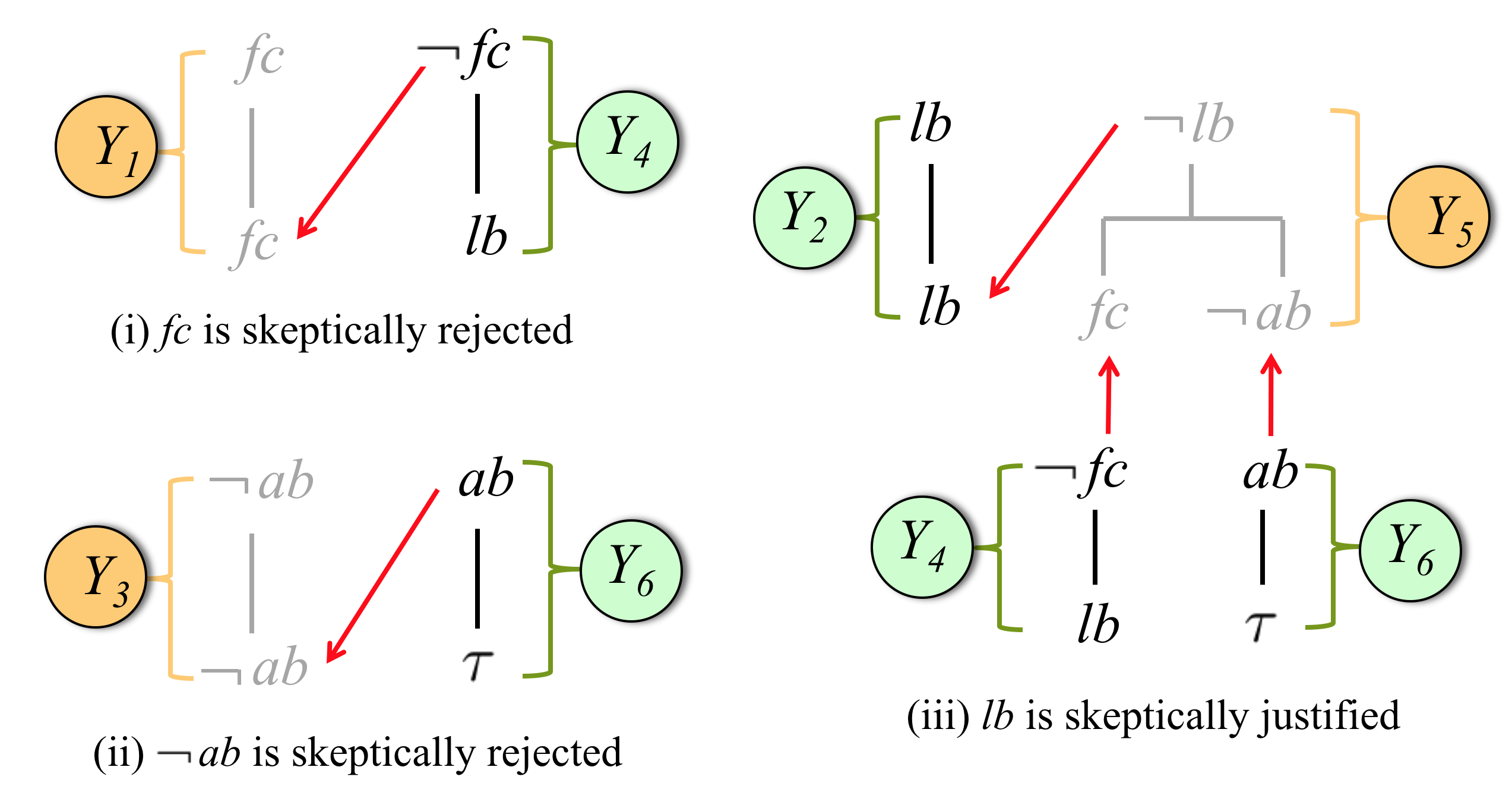}
  \caption{Explanations for skeptically accepted/rejected assumptions}
    \label{fig:ex-epistemic}
\end{figure}

Besides the above-mentioned approach for explaining why an action is selected or rejected, we may also use argument-based dialogues to provide explanations \cite{DBLP:journals/flap/CyrasFST17}.  This is left to future work. 





\section{Related Work}
The work presented in this paper concerns, in the main, how an autonomous agent makes decisions according to ethical considerations and provides an explanation for these decisions.  In this section, we discuss related work from the perspectives of value/ norm-based reasoning, argumentation-based decision making, and explanations in artificial intelligence.

Lopez-Sanchez et al. \cite{DBLP:conf/aies/SerramiaLRMWA18} pertains to the use of ``moral values" in choosing correct norms. Using deontic logic, they associate moral values with norms that exhibit them and incorporate the relative weights of these values as a factor in calculating which norms should take precedence. A correlation might be made between what is termed ``moral values" in the cited work and the concept of duties in Anderson et al's work \cite{DBLP:conf/aaai/AndersonAB17} but doing so reveals the simplistic manner in which these values are treated. In \cite{DBLP:conf/aies/SerramiaLRMWA18}, values are chosen in an arbitrary manner without the support of consideration by ethicists and are not likely to form the total order that they assume. Further, the fact that some actions may satisfy more than one value is not considered nor is the possibility of actions violating values. Also not considered is the possibility that an action satisfies or violates a duty more or less than another. Such possible combinations of different levels of a variety of satisfied and violated duties are likely to require non-linear means to resolve-- this is precisely what the principle formalized in this paper accomplishes.  Lastly, the cited work doesn't seem to address the core of what we are trying to accomplish -- providing arguments/explanations for chosen actions. 

Those who attempt to exploit deontic logic in the service of providing ethical guidance to autonomous systems, (e.g. \cite{kim2018toward}, \cite{arkoudas2005toward}) cite the transparency of the reasoning process as a benefit of such an approach. We would argue that, while a trace of deductive reasoning from premises to a conclusion may be transparent to some, it will forever remain opaque to others. We maintain that an argumentation approach to explanation may be more fruitful.


The work reported in this paper shares some similarity with the symbolic approach introduced in \cite{DBLP:conf/ijcai/ShihCD18}, in the sense that some implicit functions of the system are made explicit by using a symbolic representation. However, rather than translating the function between a set of features and a classification, we translate several types of implicit knowledge of a VDA by a logical formalism. 

Other related works are those based on argumentation. Among others, Liao et al. \cite{DBLP:journals/logcom/LiaoOTV19} introduce an argumentation-based formalism for representing prioritized norms, but do not consider the origin of these priorities, while in this paper the priority relation between the ethical consequences of different actions are learned from a set of cases, guided by the judgement of ethicists.  Cocarascu et al. \cite{Cocarascu2018} introduce an approach to construct an AAF in terms of highest ranked features. While sharing some ideas of developing a methodology of explainable AI by combining argumentation and machine learning, our approach is specific to machine ethics and connects to a different machine learning approach and has a different model of argumentation. Baum et al. \cite{DBLP:conf/isaim/BaumHS18} study the interplay of machine ethics and machine explainability by using argumentation. The idea is close to our work, but focuses on a different research setting and has a different model.  Others also focus on developing general approaches for explanation based on argumentation, e.g., \cite{DBLP:conf/aaai/FanT15}'s work on a new argumentation semantics for giving explanations to arguments in both Abstract Argumentation and Assumption-based Argumentation, and  \cite{DBLP:journals/eswa/GarciaCRS13}'s work on dialectical explanation for argument-based reasoning in knowledge-based systems, etc. However, they are not specific to ethical decision making and explanation. 

Last but not least, in the direction of explanation in artificial intelligence, there are a number of research efforts in recent years. Among them, a recent work by Tim Miller \cite{DBLP:journals/ai/Miller19} provides several  insights from the social sciences,  by considering how people define, generate, select, evaluate, and present explanations. 
 
\section{Conclusions}
In this paper, we have proposed an argumentation-based approach for representation, justification and explanation of a VDA. The contributions are as follows. 

First, we provide a formalism to represent a VDA, making explicit some implicit knowledge. This lays a foundation for the justification and explanation of reasoning and decision making in a VDA. To our knowledge, this is the first effort on providing a formalization for a VDA. 

Second, we adapt existing argumentation theory to the setting of a decision making in a VDA, such that the ethical consequences of actions and clauses of a principle can be used for decision-making and explanation. Furthermore, we go beyond the existing version of VDA, considering not only practical reasoning, but also epistemic reasoning, such that the inconsistency of knowledge of the VDA can be identified, handled and explained.

Third, unlike existing argumentation systems where formal rules are designed in advance, in our approach, action rules and principle rules for practical reasoning are generated and updated at run time in terms of an action matrix and a principle. Thanks to the graphic nature of an AAF, when the system becomes more complex, there exist efficient approaches to handle the dynamics of the system, e.g., \cite{DBLP:journals/ai/LiaoJK11} and \cite{DBLP:books/daglib/0033440}. 

Besides these technical contributions, methodologically, this paper provides a novel approach for combining symbolic approaches and sub-symbolic approaches, in the sense that the features learned from data could be used to build the rules for reasoning. In this paper, the duty satisfaction/violation vectors and the principle are exploited to build the knowledge for reasoning, decision-making and explanation. 

Due to these contributions, some benefits can be obtained. Clearly, formal justification and explanation of the behavior of autonomous systems enhances the transparency of such systems.  Further, autonomous systems that can argue formally for their actions are more likely to engender trust in their users than systems without such a capability. That principle-based systems such as the one detailed in this paper and others (e.g. \cite{VANDERELST201856},\cite{sarathyetal2017coginfocom}) seem to lend themselves readily to explanatory mechanisms adds further support for the adoption of principles as a formalism to ensure the ethical behavior of autonomous systems. 

Concerning future work, first, we have not identified nor formally represented the relation between a principle and a set of cases from which the principle is learned. Doing so is likely to provide further information that explains why an action is chosen in a given situation. Second, in the existing version of VDA \cite{DBLP:journals/pieee/Anderson19}, multi-agent interaction \cite{DBLP:conf/agents/BroersenDHHT01,DBLP:journals/ai/AtkinsonB18,handbooknms} has been considered. The addition of such extensions to the VDA will serve to extend its capabilities. Third, concerning explanations, it could be interesting to further develop our approach by using argument-based dialogues, and the insights from the social sciences, as pointed in \cite{DBLP:journals/ai/Miller19}.

\section*{Acknowledgment}
The authors are grateful to the anonymous reviewers of PRIMA2019 for their helpful comments. All the comments are carefully taken into consideration, and the paper is revised accordingly. 

\bibliographystyle{splncs}
\bibliography{AIEs2019}

\end{document}